\documentclass[11pt,a4paper]{article}

\usepackage[utf8]{inputenc}
\usepackage[T1]{fontenc}

\usepackage{tabularx}
\usepackage[numbers]{natbib}
\usepackage[pdftex]{graphicx}
\usepackage[pdftex,linkcolor=black,pdfborder={0 0 0}]{hyperref}
\usepackage{url}
\usepackage{calc}
\usepackage{enumitem}
\usepackage{amsmath}
\usepackage{amsfonts}
\usepackage{authblk}
\usepackage{amsthm}
\usepackage{amssymb}
\usepackage{nicefrac}
\usepackage{algorithm}
\usepackage{algpseudocode}
\usepackage{subfig}
\usepackage{float}
\usepackage{booktabs}
\usepackage{adjustbox}
\usepackage{multirow}
\usepackage{multicol}
\usepackage{makecell}
\usepackage{textcomp}
\usepackage{xcolor}
\usepackage[english]{babel}

\newtheorem{theorem}{Theorem}[section]
\newtheorem{lemma}[theorem]{Lemma}
\newtheorem{proposition}[theorem]{Proposition}

\newcommand{\Ur}{\mathcal{U}_r}

\newcommand{\argmin}{\mathop{\mathrm{argmin}}}
\newcommand{\argmax}{\mathop{\mathrm{argmax}}}
\newcommand{\Ht}{\mathcal{H}}
\newcommand{\Nt}{\mathcal{N}}
\newcommand{\Et}{\mathcal{E}}

\usepackage[
  a4paper,
  lmargin=0.1666\paperwidth,
  rmargin=0.1666\paperwidth,
  tmargin=0.1111\paperheight,
  bmargin=0.1111\paperheight
]{geometry}

\usepackage[all]{nowidow}
\usepackage[protrusion=true,expansion=true]{microtype}

\title{Connected by Construction: Learning Tractable Near-Tour Marginals for Travseling Salesman Problems}

\author[1]{Ke Sun}
\author[1]{Xinyuan Zhang}
\author[1,*]{Xinwu Qian}

\affil[1]{Department of Civil and Environmental Engineering, Rice University, Houston, TX}
\affil[*]{Corresponding author}

\date{\today}

\begin{document}

\maketitle

\begin{abstract}
Learning-based methods for the traveling salesman problem (TSP) are often evaluated through the tours produced after decoding or search, but the learned object itself frequently lives in a surrogate space such as heatmaps, assignments, construction policies, or search-guidance scores. This hides the fundamental question: what Hamiltonian structure has actually been learned before decoding? In this study, we directly answer this question by learning TSP through a structurally meaningful latent object, rather than leaving most of the Hamiltonian structure to the final decoding stage.  Based on a connected-by-construction rooted $1$-tree Gibbs family, we propose an end-to-end unsupervised learning pipeline called \emph{C2TSP}. The pipeline learns residual edge perturbations from unbiased TSP cost through implicit differentiation. For structural correction, a smoothed Held--Karp layer restores expected degree balance, while certificate-guided sharpening further pushes the connected distribution toward more tour-like structures. 
Experiments show that C2TSP yields strong decoding performance while preserving interpretable structural information. Ablations further verify that edge perturbation and certificate-guided sharpening jointly improve both tour cost and tour-like structure.
\end{abstract}

\section{Introduction}\label{sec:introduction}

Machine learning methods have been increasingly used for combinatorial optimization because they can provide fast approximate solutions while maintaining competitive empirical performance. As a canonical combinatorial optimization problem, the traveling salesman problem (TSP) has received particular attention. A feasible TSP solution must satisfy both local and global structural constraints: every node has degree two, and the selected edges must form one connected cycle. However, many learning-based TSP methods do not represent this global connected edge structure directly in the learned object. Autoregressive construction models generate tours sequentially and enforce feasibility procedurally through masking and decoding rules~\citep{vinyals2015pointer,bello2016neural,kool2019attention,kwon2020pomo}. Heatmap-based methods predict edge scores and then rely on repair, sampling, local search, or tree search to obtain a valid tour~\citep{fu2021generalize,qiu2022dimes,min2023utsp,sun2023difusco,li2024fast}. Assignment-space relaxations, such as Sinkhorn-based methods, provide smooth permutation-like objects but do not directly model the edge-adjacency structure that determines tour cost~\citep{mena2018gumbelsinkhorn,paulus2020gradient,min2026graph}. Hybrid methods such as NeuroLKH use learned signals to guide strong handcrafted heuristics~\citep{xin2021neurolkh}. Although these methods have achieved strong empirical results with affordable inference time, two limitations remain: they do not provide a tractable global distribution over connected combinatorial objects, and their final performance can be tightly coupled with downstream decoding or search, making the learned representation and the post-processing effect difficult to separate~\citep{xia2024position,bother2022s}.

In contrast, we propose a connected-by-construction representation based on the rooted $1$-tree family. After fixing a root node, a rooted $1$-tree consists of a spanning tree on the non-root nodes together with exactly two edges incident to the root. Therefore, every latent configuration is globally connected by construction and satisfies the root degree constraint exactly. The remaining structural defect is the non-root degree mismatch: a rooted $1$-tree becomes a Hamiltonian cycle when every non-root node also has degree two. Unlike a Gibbs distribution over Hamiltonian cycles, whose partition function is intractable, the rooted $1$-tree Gibbs family admits an exact factorization, which enables exact marginal computation and expected-cost training. This connects our method to the classical Held--Karp relaxation, where node penalties modify edge costs and minimum $1$-trees provide strong TSP lower bounds~\citep{held1970traveling,held1971traveling}. Prior work has explored tree-based or Held--Karp-related structures for TSP from theoretical, heuristic, or learning perspectives~\citep{parjadis2023learning,gharan2011randomized,genova2017experimental}. Our goal is to turn this connected structure into an end-to-end differentiable latent representation for learning near-tour marginals.

We build an end-to-end unsupervised learning pipeline from this connected-by-construction rooted $1$-tree latent family, called \emph{C2TSP}. The model first predicts residual edge perturbations that tilt the rooted $1$-tree Gibbs distribution toward low-cost structures guided by the expected TSP cost. To correct the remaining non-root degree defect, we introduce a smoothed Held--Karp equilibration layer, motivated by its compatibility with the rooted $1$-tree prior and the classical Held--Karp lower-bound structure~\citep{held1970traveling}. This layer solves for node-additive dual variables so that the expected degree of every non-root node equals two under the dual-modified rooted $1$-tree distribution, and it is trained end-to-end by differentiating through the equilibrium map, following the general paradigm of implicit optimization layers~\citep{amos2017optnet,agrawal2019diffcp,bai2019deq,blondel2022modular}. After equilibration, the residual error is no longer a connectivity failure, but the non-tour mass inside a connected, degree-balanced latent family. We control this residual defect using a certificate that upper-bounds non-tour mass by the total variance of the non-root degrees. A certificate-guided residual refinement then reduces this non-tour mass, followed by another re-equilibration step. Finally, decoding is performed by sampling from the learned edge marginals rather than by search-based post-processing.

The contributions of this paper are as follows.
\begin{itemize}
    \item We introduce a tractable rooted $1$-tree representation for TSP that ensures global connectivity by construction while admitting exact marginal computation for expected-cost training.
    \item We develop C2TSP, an end-to-end unsupervised learning pipeline in which a GNN predicts residual edge perturbations from the expected TSP cost and a smoothed Held--Karp equilibration layer restores expected degree balance.
    \item We derive a non-tour mass certificate and use it to guide residual refinement, pushing the learned connected distribution toward more tour-like structures.
    \item We empirically show that the proposed rooted $1$-tree representation yields strong pure-decoding performance, while several baselines benefit more from additional local search. Further ablations show how edge perturbation and certificate-guided sharpening shift the connected distribution toward more tour-like and lower-cost structures.
\end{itemize}
\section{Methodology}\label{sec:methodology}

\subsection{Problem definition}\label{subsec:problem}
\paragraph{Exact Hamiltonian Gibbs model.} Let $V=\{1,\ldots,n\}$ and let $\Et$ be the edge set of the complete undirected graph. Each Hamiltonian cycle $H\in\Ht$ has edge-incidence vector $x_H\in\{0,1\}^{|\Et|}$, and $D\in\mathbb R^{|\Et|}$ denotes the edge-cost vector. A smooth exact model is the Hamiltonian Gibbs law
\begin{align}
    q_D^{\Ht}(H)
    =
    \frac{1}{Z_{\Ht}(D)}
    \exp\!\left(
        -\frac{1}{\tau}\langle D,x_H\rangle
    \right),
    \qquad H\in\Ht ,
    \label{eq:ham-gibbs}
\end{align}
where $Z_{\Ht}(D)$ is the Hamiltonian partition function. This model is intractable because evaluating $Z_{\Ht}(D)$ and differentiating $\log Z_{\Ht}(D)$ require summing over exponentially many Hamiltonian cycles.

\paragraph{Connected by construction: the rooted $1$-tree surrogate.}
Since exact Hamiltonian Gibbs inference is intractable, we pursue a tractable surrogate with structural guarantee, that (1) preserves core tour structure, (2) certifies the structural gap it leaves, and (3) learns to reduce that gap under the equilibration guarantees developed below. Note that a Hamiltonian cycle is always a connected spanning edge set with degree two at every node. We preserve connectivity by using rooted $1$-trees, and address the remaining degree condition via a smoothed Held--Karp (HK) equilibration layer.

Fix a root $r\in V$ and write $\bar V:=V\setminus\{r\}$. Let $\Ur$ denote the family of rooted $1$-trees, namely spanning trees on $\bar V$ together with exactly two edges incident to $r$. Every $U\in\Ur$ is connected by construction, satisfies $d_r(U)=2$, and has edge-incidence vector $x_U\in\{0,1\}^{|\Et|}$. And it is easy to verify that $\mathcal H\subseteq \mathcal U_r$. For any edge parameter $\eta\in\mathbb R^{|\Et|}$, we define the rooted $1$-tree Gibbs law
\begin{align}
    q_\eta(U)
    =
    \frac{1}{Z_r(\eta)}
    \exp\!\left(
        -\frac{1}{\tau}\langle \eta,x_U\rangle
    \right),
    \qquad U\in\Ur ,
    \label{eq:rooted-gibbs}
\end{align}
and its edge marginal
\begin{align}
    \mu(\eta)
    :=
    \mathbb E_{q_\eta}[x_U].
    \label{eq:rooted-marginal}
\end{align}

Lemma~\ref{lem:tractable-rooted-map} shows that the tractable marginal~\eqref{eq:rooted-gibbs} is tractable with exact partition function. The remaining task is to formalize the edge
parameters $\eta$ that preserve the degree-two structure and reduce structural deficiency inside $\Ur$.

\begin{lemma}[Tractable rooted $1$-tree marginal]
\label{lem:tractable-rooted-map}
The partition function $Z_r(\eta)$ in~\eqref{eq:rooted-gibbs} factorizes
into a weighted spanning-tree partition function on $\bar V$ and a
two-edge root-selection normalizer. Consequently, $Z_r(\eta)$, the edge
marginal $\mu(\eta)$, and all degree moments are computable exactly.
\end{lemma}

\subsection{Held--Karp equilibration and structural-gap certificate}
\label{subsec:hk-certificate}

Given the $1$-tree Gibbs family, the remaining Hamiltonian condition is to enforce degree two at the non-root nodes. Since the Gibbs law is parameterized by edge costs, the dual prices for the non-root degree constraints need to be represented as an edge-additive field. For $\lambda\in\mathbb R^{n-1}$ indexed by $\bar V$, define the node-to-edge
lift $A:\mathbb R^{n-1}\to\mathbb R^{|\Et|}$ by
\begin{align}
(A\lambda)_{ij}
=
\begin{cases}
\lambda_i+\lambda_j, & i,j\in\bar V,\\
\lambda_k, & \{i,j\}=\{r,k\},\ k\in\bar V .
\end{cases}
\label{eq:lift}
\end{align}
Then, for every $U\in\Ur$,
\begin{align}
    \langle A\lambda,x_U\rangle
    =
    \sum_{i\in\bar V}\lambda_i d_i(U).
    \label{eq:lift-degree-identity}
\end{align}
Thus $A\lambda$ is exactly the edge-additive representation of the non-root degree prices.

For a generic edge cost field $c\in\mathbb R^{|\Et|}$, consider the
entropy-regularized degree-correction problem for the $1$-tree Gibbs family as 
\begin{align}
    \min_{p\in\Delta(\Ur)}
    \quad&
    \sum_{U\in\Ur}p(U)\langle c,x_U\rangle
    +
    \tau\sum_{U\in\Ur}p(U)\log p(U)
    \label{eq:hk-primal}
    \\
    \mathrm{s.t.}\quad&
    \sum_{U\in\Ur}p(U)d_i(U)=2,
    \qquad i\in\bar V ,
    \nonumber
\end{align}
where $\Delta(\Ur)$ denotes the probability simplex over rooted $1$-trees. This problem asks for the lowest free-energy distribution over connected rooted $1$-trees whose expectation of non-root degree match the Hamiltonian target. Substituting the lift identity~\eqref{eq:lift-degree-identity} into the Lagrangian collapses the degree multipliers into a tilted edge cost $c+A\lambda$, which keeps the partition function in the rooted 1-tree family denoted by $q_{c+A\lambda}$, as stated below. 

\begin{theorem}[Smoothed Held--Karp equilibrium]
\label{thm:hk-equilibrium}
The Lagrange dual of~\eqref{eq:hk-primal} is
\begin{align}
    \max_{\lambda\in\mathbb R^{n-1}}
    \Phi_\tau(c,\lambda),
    \qquad
    \Phi_\tau(c,\lambda)
    :=
    -\tau\log Z_r(c+A\lambda)
    -
    2\mathbf 1^\top\lambda .
    \label{eq:hk-dual}
\end{align}
Moreover,
\begin{align}
    \frac{\partial \Phi_\tau(c,\lambda)}{\partial \lambda_i}
    &=
    \mathbb E_{q_{c+A\lambda}}[d_i(U)]-2,
    \qquad i\in\bar V,
    \label{eq:hk-gradient}
    \\
    \frac{\partial^2 \Phi_\tau(c,\lambda)}
    {\partial\lambda_i\partial\lambda_j}
    &=
    -\frac{1}{\tau}
    \operatorname{Cov}_{q_{c+A\lambda}}
    \bigl(d_i(U),d_j(U)\bigr),
    \qquad i,j\in\bar V .
    \label{eq:hk-hessian}
\end{align}
Hence $\Phi_\tau(c,\lambda)$ is concave in $\lambda$, and any stationary
maximizer $\lambda^\star(c)$ satisfies
\begin{align}
    \mathbb E_{q_{c+A\lambda^\star(c)}}[d_i(U)]
    =
    2,
    \qquad i\in\bar V .
    \label{eq:hk-degree-balance}
\end{align}
\end{theorem}

The proof is given in Appendix~\ref{app:proof-hk-equilibrium}. Theorem~\ref{thm:hk-equilibrium} shows that the degree multipliers enter each $1$-tree energy through $\langle A\lambda,x_U\rangle-2\mathbf 1^\top\lambda$, and that optimizing over $p\in\Delta(\Ur)$ recovers the Gibbs law $q_{c+A\lambda}$. The HK layer is therefore the implicit map
\[
    c
    \mapsto
    \lambda^\star(c)
    \quad\text{defined by}\quad
    \nabla_\lambda\Phi_\tau(c,\lambda^\star(c))=0 .
\]
This equilibration removes the degree defect of the connected surrogate. The remaining structural gap is the residual probability mass assigned to degree-balanced rooted $1$-trees outside Hamiltonian-cycle support. The next result certifies this residual gap. 

\paragraph{Structural-gap certificate.}
Let $\Nt:=\Ur\setminus\Ht$, and $\gamma(q):=q(\Nt)$
denote the set of non-tour $1$ trees and their mass, respectively.  HK equilibration enforces degree balance only in expectation: \(\mathbb E_q[d_i(U)]=2\) for all \(i\in\bar V\). But this does not force \(q\) to be supported on \(\Ht\), as non-root degree may still fluctuate across rooted $1$-trees but still having mean degree 2. The residual gap is therefore bounded by degree variance and can serve as a certificate that we can formally establish using same second-order statistic governing the HK Hessian~\eqref{eq:hk-hessian}, as stated in the theorem below. 

\begin{theorem}[Certificate for residual non-tour mass]
\label{thm:certificate}
For \(q\in\Delta(\Ur)\) at the HK equilibrium solution, the non-tour mass satisfies
\begin{align}
    \gamma(q)
    \le
    \frac12
    \sum_{i\in\bar V}
    \operatorname{Var}_q(d_i(U)).
    \label{eq:non-tour-certificate}
\end{align}
For the HK-equilibrated Gibbs law
\(q^\star=q_{c+A\lambda^\star(c)}\), this becomes
\begin{align}
\label{eq:hessian-certificate}
    \gamma(q^\star)
    \le
    \frac12
    \sum_{i\in\bar V}
    \operatorname{Var}_{q^\star}(d_i(U))=
    \frac12
    \operatorname{tr}
    \operatorname{Cov}_{q^\star}(d(U),d(U))
    =
    -\frac{\tau}{2}
    \operatorname{tr}
    \nabla_{\lambda\lambda}^2
    \Phi_\tau(c,\lambda^\star(c)).
\end{align}
\end{theorem}

The proof is given in Appendix~\ref{app:proof-certificate}.
Thus the same covariance
that defines the HK Hessian in Theorem~\ref{thm:hk-equilibrium} also certifies the residual structural gap left
after HK equilibration. As a consequence, the $1$-tree Gibbs family not only provides tractable marginal, but also supplies the analytical structure that allows us to enforce the structure primitives to become Hamiltonian cycles. 

\subsection{Oracle residual edge field and learned re-equilibration}
\label{subsec:oracle-edge-field}

The certificate in Theorem~\ref{thm:certificate} measures the structural gap left after HK equilibration. This gap cannot be removed by another node-price update. To see this, following~\ref{eq:hk-primal}, for any $\beta\in\mathbb R^{n-1}$,
\begin{equation}\label{eq:hk-node-shift-equivariance}
    \Phi_\tau(c+A\beta,\lambda)
    =
    \Phi_\tau(c,\lambda+\beta)
    +
    2\mathbf 1^\top\beta .
\end{equation}

Thus a node-additive edge shift only reparametrizes the HK price and leaves the equilibrated Gibbs law and its edge marginals unchanged. Hence the residual edge correction is in general not unique and can be absorbed into the HK price. We need to select the identifiable representative by removing this node-price component and retaining the part orthogonal to all node-additive edge fields $A\lambda$.

To identify the oracle form of this residual, suppose the omitted tour-valid structure inside the rooted $1$-tree family is represented at the marginal level by inequalities $C\mu\le h$,
where $C$ and $h$ are exponential in size in general, due to the combinatorial structure of the problem. Now extending~\eqref{eq:hk-primal} with these constraints gives the entropy-regularized oracle recovery problem
\begin{align}
    \min_{p\in\Delta(\Ur)}
    \quad&
    \sum_{U\in\Ur}
    p(U)\langle D,x_U\rangle
    +
    \tau
    \sum_{U\in\Ur}p(U)\log p(U)
    \label{eq:oracle-recovery}
    \\
    \mathrm{s.t.}\quad&
    \mathbb E_p[d_i(U)]=2,
    \qquad i\in\bar V,
    \qquad C\,\mathbb E_p[x_U]\le h .
    \nonumber
\end{align}

\begin{proposition}[Oracle orthogonal residual edge field]
\label{prop:oracle-edge-field}
Let \((p^\star,\bar\lambda^\star,\nu^\star)\) satisfy the KKT conditions of
\eqref{eq:oracle-recovery}, where \(\bar\lambda^\star\) are the degree
multipliers and \(\nu^\star\ge0\) are the multipliers of \(C\mu\le h\). Define
the raw oracle edge field $\widetilde\Gamma^\star:=C^\top\nu^\star$.
Let
\[
    \widetilde\Gamma^\star
    =
    A\beta^\star+\Gamma^\star,
    \qquad
    \Gamma^\star\in\operatorname{Im}(A)^\perp ,
\]
be its decomposition into a node-additive component and an orthogonal residual
component, and define the absorbed HK price $\lambda^\star:=\bar\lambda^\star+\beta^\star$.
Then \(\lambda^\star\) is an HK price for the shifted cost field
\(D+\Gamma^\star\):
$\lambda^\star
\in
\argmax_{\lambda}
\Phi_\tau(D+\Gamma^\star,\lambda)$.

Consequently, the oracle re-equilibrated edge cost is
\[
    \eta^\star:=D+\Gamma^\star+A\lambda^\star .
\]

\end{proposition}

The proof is given in Appendix~\ref{app:proof-oracle-edge-field}.
Proposition~\ref{prop:oracle-edge-field}
shows that the raw dual effect \(C^\top\nu^\star\) separates into a node-additive component \(A\beta^\star\), which is absorbed into the HK price, and an orthogonal residual component \(\Gamma^\star\in\operatorname{Im}(A)^\perp\), which is the identifiable oracle edge field. The corresponding rooted \(1\)-tree Gibbs law \(q_{\eta^\star}\) then recovers the oracle solution. Since the omitted constraints are intractable, we seek to learn a residual edge field \(\Gamma_\theta\) and re-equilibrate:
\begin{align}
    \lambda^\star(\theta)
    &:=
    \lambda^\star(D+\Gamma_\theta),
    \qquad
    \eta^\star_\theta
    :=
    D+\Gamma_\theta+A\lambda^\star(\theta),
    \label{eq:learned-reequilibration}
\end{align}
Thus in this study, we use learning to supply the residual edge geometry, while rooted \(1\)-tree inference and HK equilibration are used to structure the trajectory in shaping and learning such residual. The learned law is degree-balanced after re-equilibration, and Theorem~\ref{thm:certificate} certifies its remaining structural gap.

\subsection{Unsupervised learning with certificate-directed re-equilibration}
\label{subsec:unsup-learning}

With learnable object identified as an orthogonal residual edge field, we next present the formal pipeline to train this residual that is guided by the unbiased TSP cost:
\begin{equation}
    \mathcal L_{\rm TSP}(\theta)
    =
    \frac1M
    \sum_{m=1}^M
    \langle D^m,u_\theta(D^m)\rangle ,
    \label{eq:final-unsup-loss}
\end{equation}
where \(u_\theta(D)\) is the edge marginal of the $1$-tree Gibbs family produced by the forward map discussed below. This induces an unsupervised learning framework, and more importantly, that \textit{learns from unbiased TSP cost signals} as the structural constraints are not enforced into the loss function as penalty terms. The forward process supplies the structure constraints, where HK recovery leads to expected degree balance, and certificate-directed sharpening further reduce a certified upper bound on residual non-tour mass. The backward propagation can be computed via implicit gradient with details discussed in the Appendix~\ref{app:training-implicit}.

\paragraph{Canonical HK forward map.} By Theorem~\ref{thm:hk-equilibrium}, every edge input \(c\) induces HK price \(\lambda^\star(c)\), an equilibrated edge cost \(c+A\lambda^\star(c)\), and a rooted \(1\)-tree marginal \(\mu(c+A\lambda^\star(c))\). As a result, for each instance \(D\), the canonical forward map starts with a raw edge field \(\widetilde\Gamma_\theta(D)\), followed by removing the node-price component as in Prop.~\ref{prop:oracle-edge-field}: 
\[ \Gamma_\theta^{(0)}(D) = P_\perp\widetilde\Gamma_\theta(D), \qquad P_\perp:=P_{\operatorname{Im}(A)^\perp}. \label{eq:projected-residual-learning} \] 

The projected residual then enters the HK equilibration layer. Given
$c_\theta(D)
    =
    D+\Gamma_\theta^{(0)}(D),$
one can run \(L\) HK iterations indexed by
\(\ell=0,\ldots,L-1\), producing
\[
    \lambda_{\theta,\ell}^{(0)}(D)
    \longrightarrow
    \lambda_{\theta,L}^{(0)}(D)
    \approx
    \lambda^\star(c_\theta^{(0)}(D)).
\]
The terminal HK price defines the $1$-tree marginals that are evaluated in the loss function~\ref{eq:final-unsup-loss}.
\[
    \eta_\theta^{(0)}(D)
    =
    c_\theta^{(0)}(D)+A\lambda_{\theta,L}^{(0)}(D),
    \qquad
    u_\theta^{(0)}(D)
    =
    \mu(\eta_\theta^{(0)}(D)).
    \label{eq:base-hk-forward}
\]

\paragraph{Bernoulli certificate and marginal-sharpening forward map.}
The certificate in Theorem~\ref{thm:certificate} bounds the residual non-tour mass after HK equilibration, but directly differentiating the full covariance certificate would require higher-order covariance derivatives. In practice, one can instead use a Bernoulli upper certificate that depends only on the recovered edge marginal $u$ at HK equilibrium:
\[
    \mathcal C_{\rm Ber}(u)
    :=
    \sum_{e\in\Et}w_eu_e(1-u_e),
    \qquad
    w_e:=\frac12\sum_{i\in\bar V}B_{ie},
\]
where \(Bx_U=d_{\bar V}(U)\). Its marginal derivative by is

\begin{equation}
    b_{\rm Ber}(u):=\nabla_u\mathcal C_{\rm Ber}(u)= w\odot(1-2u)
    \label{eq:ber_descent}
\end{equation}

Since \(\eta\) is an edge-cost parameter, a positive step in this direction
sharpens the marginal by penalizing low-probability edges and rewarding
high-probability edges. 

\begin{proposition}[Bernoulli certificate descent]
\label{prop:ber-cert-descent}
Let \(q_\eta\) be an HK-equilibrated rooted \(1\)-tree Gibbs law and let
\(u=\mu(\eta)\). Then $\mathcal C_{\rm Ber}(u)$ is a valid non-tour certificate $ \gamma(q_\eta)\le\mathcal C_{\rm Ber}(u)$.
Moreover, for any HK-equilibrated edge input \(c\), let
\[
    \eta(c)=c+A\lambda^\star(c),
    \qquad
    u(c)=\mu(\eta(c)).
\]
For $c_\alpha=c+\alpha b_{\rm Ber}(u(c))$ following  Eq.~\ref{eq:ber_descent} and $u_\alpha= \mu\!\left(c_\alpha+A\lambda^\star(c_\alpha)\right)$,
we have, for all sufficiently small \(\alpha>0\), $\mathcal C_{\rm Ber}(u_\alpha)<
\mathcal C_{\rm Ber}(u(c))$, whenever the re-equilibrated sharpening direction is nonzero.
\end{proposition}

Proposition~\ref{prop:ber-cert-descent} proves that the marginal derivative~\ref{eq:ber_descent} is a descent direction that minimizes the non-tour mass. This directly motivates a deterministic sharpening map applied after the canonical HK forward pass. Starting from
\((\Gamma_\theta^{(0)},u_\theta^{(0)})\), for \(k=0,\ldots,K-1\), set
\[
    \Gamma_\theta^{(k+1)}(D)
    :=
    \Gamma_\theta^{(k)}(D)
    +
    \alpha_k b_{\rm Ber}(u_\theta^{(k)}(D)).
\]
with $\alpha_k>0$ a fixed hyperparameter for sharpening aggressiveness, and re-equilibrate with
\[
    c_\theta^{(k+1)}(D)
    =
    D+\Gamma_\theta^{(k+1)}(D),
    \qquad
    \lambda_\theta^{(k+1)}(D)
    =
    \lambda^\star(c_\theta^{(k+1)}(D)),
\]
\[
    \eta_\theta^{(k+1)}(D)
    =
    c_\theta^{(k+1)}(D)+A\lambda_\theta^{(k+1)}(D),
    \qquad
    u_\theta^{(k+1)}(D)
    =
    \mu(\eta_\theta^{(k+1)}(D)).
\]
which then supply to~\eqref{eq:final-unsup-loss} the final marginal
\(u_\theta(D):=u_\theta^{(K)}(D)\). When \(K=0\), this reduces to the canonical HK forward map.

\section{Experiments}\label{sec:experiments}

We evaluate on 2-D symmetric Euclidean TSP instances, with node coordinates sampled uniformly from the unit square. We denote the dataset with $n$ nodes by TSP-$n$. This data-generation protocol follows the standard setting used in prior neural TSP work~\citep{fu2021generalize}, while we compute the optimal labels via Concorde~\citep{applegate2011traveling}. We evaluate C2TSP in two complementary regimes against representative neural and hybrid baselines, including DIFUSCO~\citep{sun2023difusco}, DIMES~\citep{qiu2022dimes}, Fast-T2T~\citep{li2024fast}, UTSP~\citep{min2023utsp}, and NeuroLKH~\citep{xin2021neurolkh}. We train all models at TSP-100 and apply zero-shot at every test size. Details are given in Appendix~\ref{app:training_details}.

\paragraph{Decoding.}
The trained network produces edge marginals $\mu$ from which a tour is recovered without any external solver. We compare at six decode levels: \textbf{(L1)} greedy degree-2 repair; \textbf{(L2)} L1 followed by 2-opt local search for $1$, $10$, or $100$ iterations; \textbf{(L3)} Gumbel perturbation of $\mu$, then decoded greedily and taken best-of-perturbation sample; \textbf{(L4)} and MAP decoder.

\paragraph{Connecting to LKH.}
Each learning method is used as a candidate-and-initial-tour oracle for LKH-3, with five integration levels of progressively more learned signal: \textbf{(H0)} vanilla LKH-3; \textbf{(H1)} a pre-computed initial tour from L1; \textbf{(H2)} top-$5$ candidates per node, ordered by $\mu$; \textbf{(H3)} the same candidates reordered by a $0.5{:}0.5$ rank fusion of $\mu$ and the network-induced reduced cost $C_{\rm mod}$; and \textbf{(H4)} the full configuration combining H1 with the H2 (H3 for C2TSP) candidate set.  H3 is C2TSP-only because they require a network-side non-root dual that the other methods do not produce.

\label{subsec:tsp50_ablation}
\begin{figure}[H]
    \centering
    \resizebox{\linewidth}{!}{%
        \includegraphics[height=0.5cm]{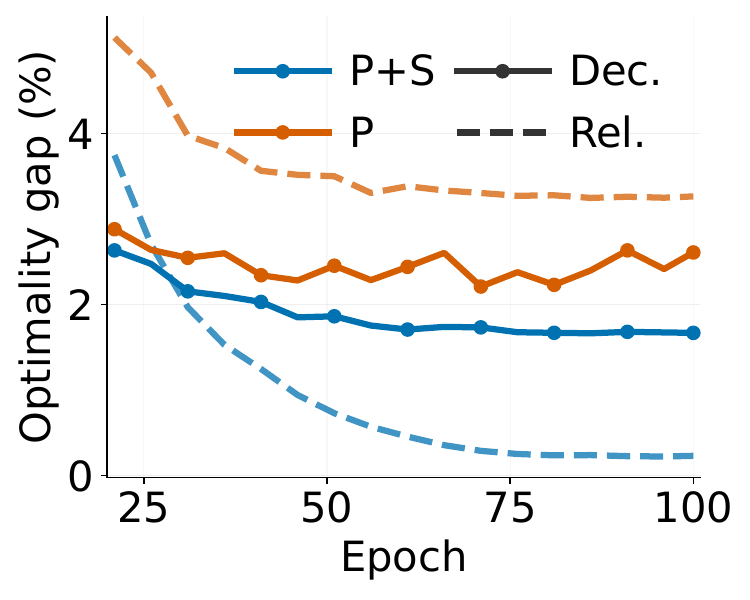}\hspace{-0.05cm}
        \includegraphics[height=0.5cm]{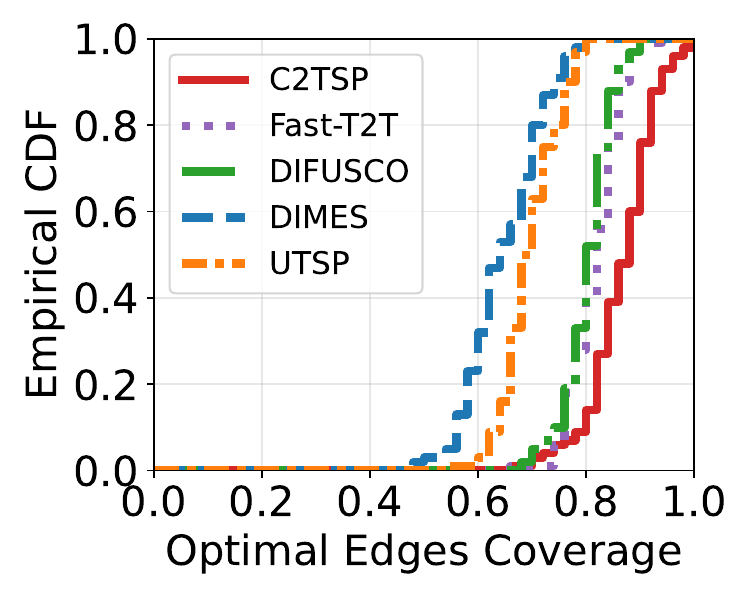}\hspace{-0.05cm}
        \includegraphics[height=0.5cm]{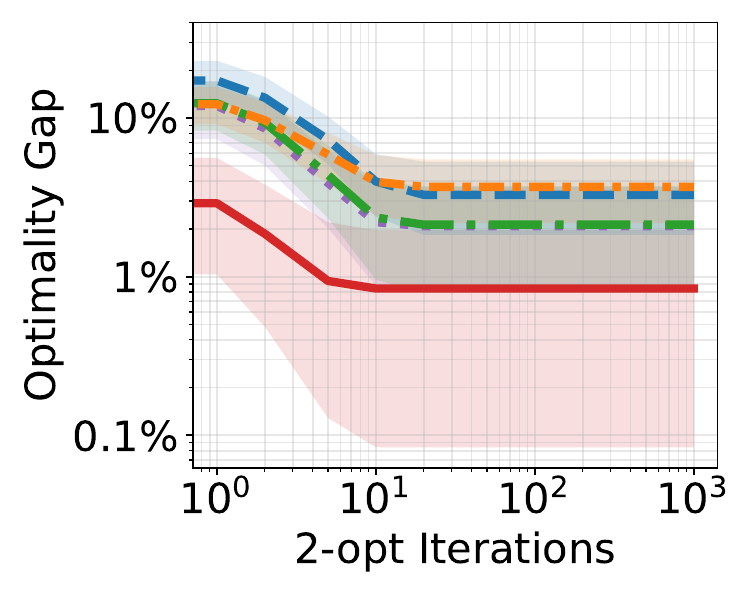}}
        \caption{Empirical analysis on TSP50. \textbf{(Left)} Training dynamics of C2TSP on TSP-100. ``P'' denotes the \emph{perturb} branch and ``P+S'' the full \emph{perturb-then-sharpen} pipeline; solid lines (``Dec.'') report the optimality gap of the decoded discrete tour, while dashed lines (``Rel.'') report the optimality gap of the continuous relaxation. \textbf{(Middle)} Empirical CDF of optimal-edge coverage when the tour is obtained by direct \texttt{argmax} over the GNN edge scores (no search). \textbf{(Right)} Optimality gap as a function of 2-opt iterations starting from the model's initial tour. Baselines construct the initial tour greedily from the heatmap; C2TSP uses MAP decoding.} \label{fig: tsp-50}
\end{figure}

\subsection{Main results}
Table~\ref{tab:decoder} reports optimality gap and per-instance wall time at every decode level. C2TSP achieves the lowest optimality gap in 20 of the 25 (size, decoder-level) cells, and is uniformly best at the heaviest local-search budget (L2$\times$100). The advantage is largest precisely where the decoder is weakest, indicating the rooted $1$-tree distribution already encodes tour-like structure before any post-processing. The remaining best-cells go to DIMES at $n\!\ge\!500$ under low-search regimes, where its near-scale-invariant heatmap edges out our marginal-based decoders. Within C2TSP, the MAP decoder (L4) is competitive with L2$\times$1 at small $n$ but degrades at $n\!\ge\!500$, indicating that the soft marginal $\mu$ carries mass beyond what the single MAP $1$-tree retains.
\begin{table}[h]
\centering
\setlength{\tabcolsep}{3.5pt}
\renewcommand{\arraystretch}{0.90}
\caption{Decoder results. Optimality gap (\%) and per-instance wall time in seconds (in parentheses).}
\label{tab:decoder}
\begin{tabular}{llrrrrr}
\toprule
Method & Level & TSP50 & TSP100 & TSP200 & TSP500 & TSP1000 \\
\midrule
\multirow{6}{*}{C2TSP}
  & L1            & \textbf{6.12} (0.14) &  \textbf{9.44} (0.15) & \textbf{12.71} (0.20) & 17.73 (0.97) & 21.98 (7.34) \\
  & L2 $\times 1$ & \textbf{3.73} (0.14) &  \textbf{7.03} (0.15) & \textbf{10.32} (0.19) & \textbf{14.97} (0.93) & 19.77 (7.19) \\
  & L2 $\times 10$& \textbf{1.44} (0.14) &  \textbf{1.97} (0.15) &  \textbf{3.25} (0.20) &  \textbf{7.84} (0.94) & 13.06 (7.19) \\
  & L2 $\times 100$ & \textbf{1.43} (0.14) & \textbf{1.90} (0.15) & \textbf{2.21} (0.20) & \textbf{2.75} (0.94) & \textbf{3.83} (7.21) \\
  & L3            & \textbf{3.09} (0.14) &  \textbf{5.11} (0.17) &  \textbf{7.97}(0.29) & \textbf{13.54} (1.61) & 20.35 (10.08) \\
  & L4            & 2.36 (0.14) &  4.83 (0.17) &  8.74 (0.29) & 15.38 (1.63) & 21.96 (11.98) \\
\midrule
\multirow{5}{*}{DIFUSCO}
  & L1            & 17.24 (0.07) & 22.60 (0.35) & 36.34 (0.96) & 82.08 (5.97) & 141.05 (24.18) \\
  & L2 $\times 1$ & 12.93 (0.07) & 19.07 (0.35) & 33.92 (0.96) & 80.56 (5.94) & 139.96 (24.06) \\
  & L2 $\times 10$&  2.71 (0.07) &  6.78 (0.35) & 22.43 (0.96) & 72.08 (5.94) & 133.36 (24.06) \\
  & L2 $\times 100$ & 2.48 (0.07) &  3.73 (0.35) &  5.74 (0.98) & 32.86 (5.96) &  94.93 (24.09) \\
  & L3            &  9.63 (0.07) & 15.41 (0.37) & 36.39 (1.05) & 97.00 (6.50) & 172.62 (26.40) \\
\midrule
\multirow{5}{*}{Fast-T2T}
  & L1            & 17.63 ($<\!.01$) & 14.31 (0.01) & 21.36 (0.02) & 41.20 (0.15) & 74.62 (0.62) \\
  & L2 $\times 1$ & 12.55 ($<\!.01$) & 11.04 (0.01) & 18.60 (0.02) & 39.23 (0.12) & 73.20 (0.49) \\
  & L2 $\times 10$&  2.60 ($<\!.01$) &  2.80 (0.01) &  9.01 (0.02) & 31.84 (0.12) & 67.50 (0.49) \\
  & L2 $\times 100$ & 2.46 ($<\!.01$) & 2.44 (0.01) & 4.39 (0.03) & 10.47 (0.14) & 42.72 (0.52) \\
  & L3            &  8.46 (0.01) &  7.77 (0.03) & 24.95 (0.11) & 74.42 (0.70) & 136.48 (2.96) \\
\midrule
\multirow{5}{*}{DIMES}
  & L1            & 13.72 ($<\!.01$) & 16.08 ($<\!.01$) & 16.18 (0.01) & \textbf{16.23} (0.04) & \textbf{15.84} (0.16) \\
  & L2 $\times 1$ & 10.05 ($<\!.01$) & 13.29 ($<\!.01$) & 14.28 ($<\!.01$) & \textbf{14.97} (0.01) & \textbf{14.96} (0.03) \\
  & L2 $\times 10$&  2.66 ($<\!.01$) &  4.28 ($<\!.01$) & 6.61 ($<\!.01$) &  9.67 (0.01) & \textbf{11.23} (0.04) \\
  & L2 $\times 100$ & 2.55 ($<\!.01$) & 3.25 ($<\!.01$) & 3.54 (0.01) & 3.73 (0.02) & 3.95 (0.07) \\
  & L3            &  7.87 (0.01) &  9.66 (0.02) & 11.58 (0.09) & 15.12 (0.57) & \textbf{19.38 }(2.41) \\
\midrule
\multirow{5}{*}{UTSP}
  & L1            & 16.56 ($<\!.01$) & 20.59 ($<\!.01$) & 24.43 ($<\!.01$) & 31.90 (0.02) & 39.21 (0.10) \\
  & L2 $\times 1$ & 12.54 ($<\!.01$) & 17.62 ($<\!.01$) & 22.45 ($<\!.01$) & 30.80 ($<\!.01$) & 38.54 ($<\!.01$) \\
  & L2 $\times 10$&  4.28 ($<\!.01$) &  8.40 ($<\!.01$) & 14.57 ($<\!.01$) & 25.35 ($<\!.01$) & 34.77 (0.01) \\
  & L2 $\times 100$ & 4.02 ($<\!.01$) & 5.66 ($<\!.01$) & 7.19 (0.01) & 9.98 (0.02) & 18.88 (0.04) \\
  & L3            & 60.95 (0.01) & 125.21 (0.02) & 220.61 (0.08) & 414.43 (0.49) & 633.80 (2.03) \\
\bottomrule
\end{tabular}
\end{table}

\begin{figure}[H]
    \centering
    \includegraphics[width=0.85\linewidth]{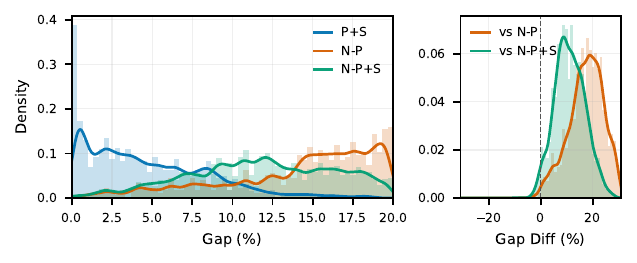}
    \caption{TSP50 ablation of edge perturbation and certificate-guided sharpening. \textbf{Left}: sampled optimality-gap distributions for the full pipeline (P+S), the variant without edge perturbation but with sharpening (N-P+S), and the variant without both perturbation and sharpening (N-P). \textbf{Right}: instance-wise gap comparisons between the full pipeline and the no-perturbation variants.}\label{fig:tsp50_ablation_stage1_two_column}
\end{figure}

Table~\ref{tab:lkh-gap} reports gap and median LKH wall time. At H1, every method stays within $\sim\!0.5\,\text{\textpertenthousand}$ of vanilla LKH-3 (H0): a learned initial tour is largely uninformative, because LKH's local search quickly escapes any mediocre start. The picture inverts at H2 and H4, where the heatmap drives LKH's candidate machinery, and DIFUSCO, DIMES, and Fast-T2T degrade by one to three orders of magnitude at $n\!\ge\!500$. C2TSP is the only learned model that survives this deeper integration, remaining within $\sim\!1\,\text{\textpertenthousand}$ of vanilla LKH at every size. Figure~\ref{fig: tsp-50} (\textbf{Middle} \& \textbf{Right}) indicates that the common thread across both main results is that the rooted $1$-tree representation preserves global connected structure by construction, so local-sweep procedures (whether 2-opt or LKH-3) inherit a globally consistent starting point and need only repair local degree defects.

\begin{table}[h]
\centering
\caption{LKH-integration results. Optimality gap ($\text{\textpertenthousand}$) and median LKH wall time per instance (s).}
\label{tab:lkh-gap}

\setlength{\tabcolsep}{3pt}
\renewcommand{\arraystretch}{1.05}
\begin{tabular}{llrrrrr}
\toprule
Level & Method & TSP100 & TSP200 & TSP500 & TSP1000 & TSP2000 \\
\midrule
H0 & LKH-3 & 0.38 (0.11) & 0.22 (0.49) & 0.36 (2.32) & 0.39 (7.49) & 0.02 (31.23) \\
\midrule
\multirow{6}{*}{H1}
   & C2TSP      & \textbf{0.18} (0.07) & 0.41 (0.29) & 0.20 (2.49) & 0.46 (5.89) & \textbf{-0.07} (26.73) \\
   & DIFUSCO   & 0.29 (0.07) & \textbf{0.34} (0.28) & 0.19 (2.54) & 0.40 (7.98) & OOM \\
   & Fast-T2T & 0.32 (0.08) & 0.64 (0.39) & 0.23 (1.99) & 0.39 (7.24) & 0.06 (27.93) \\
   & DIMES     & 0.44 (0.10) & 0.73 (0.32) & 0.19 (1.84) & \textbf{0.29} (6.39) & 0.24 (21.87) \\
   & UTSP      & 0.27 (0.11) & 0.61 (0.27) & \textbf{0.16} (2.16) & 0.41 (6.46) & 0.12 (24.73) \\
\midrule
\multirow{5}{*}{H2}
   & C2TSP      & 0.11 (0.07) & \textbf{0.13} (0.30) & \textbf{0.20} (2.82) & \textbf{0.53} (8.70) & \textbf{0.71} (37.77) \\
   & DIFUSCO   & 0.88 (0.08) & 4.91 (0.22) & 50.14 (1.87) & 130.10 (5.57) & OOM \\
   & Fast-T2T & \textbf{0.02} (0.09) & 0.76 (0.41) & 10.99 (1.99) & 36.46 (6.40) & 80.65 (20.75) \\
   & DIMES     & 3.24 (0.12) & 6.18 (0.29) & 8.27 (1.80) & 9.21 (6.77) & 6.56 (24.03) \\
   & UTSP      & 145.13 (0.09) & 211.09 (0.15) & 217.26 (1.57) & 185.35 (6.11) & 142.49 (19.77) \\
\midrule
\multirow{2}{*}{H3}
   & C2TSP      & \textbf{0.11} (0.06) & \textbf{0.04} (0.29) & \textbf{0.18} (2.49) & \textbf{0.52} (9.20) & \textbf{0.56} (37.23) \\
   & NeuroLKH  & 0.63 (0.06) & 0.53 (0.21) & 23.11 (1.01) & 38.38 (4.82) & 40.66 (27.04) \\
\midrule
\multirow{5}{*}{H4}
   & C2TSP      & 0.11 (0.07) & \textbf{0.01} (0.30) & \textbf{0.20} (2.25) & \textbf{0.58} (8.41) & \textbf{0.64} (36.87) \\
   & DIFUSCO   & 1.01 (0.08) & 2.50 (0.24) & 50.53 (1.76) & 135.08 (5.50) & OOM \\
   & Fast-T2T & \textbf{0.04} (0.09) & 0.87 (0.44) & 9.16 (2.07) & 34.85 (6.59) & 82.91 (21.11) \\
   & DIMES     & 3.57 (0.12) & 4.11 (0.32) & 4.33 (1.95) & 4.23 (7.66) & 3.55 (27.26) \\
   & UTSP      & 65.56 (0.10) & 105.35 (0.15) & 131.00 (1.57) & 118.76 (6.41) & 102.62 (21.77) \\
\bottomrule
\end{tabular}
\end{table}

\subsection{Ablation and structural analysis}
\paragraph{Component ablation.} Figure~\ref{fig:tsp50_ablation_stage1_two_column} (\textbf{Left}) shows that removing edge perturbation substantially degrades decoded performance: the average gap increases from $1.55\%$ to $12.74\%$, with many long-tail instances exceeding $20\%$. Figure~\ref{fig:tsp50_ablation_stage1_two_column} (\textbf{Right}) further show that the full pipeline improves over the no-perturbation variants on almost all test instances. Table~\ref{tab:tsp50_smooth_hk_metrics} explains this gap from a structural perspective. Without learned perturbation, the smoothed rooted $1$-tree distribution has weak edge concentration and low recovery of optimal tour edges. Sharpening alone improves these metrics, but remains far from the learned variants. In contrast, perturbation-only training already produces substantially more tour-like marginals and much better decoded tours. Figure~\ref{fig: tsp-50} (\textbf{Left}) shows a complementary training view. Without sharpening, the relaxed rooted $1$-tree objective converges quickly but the decoded tour quality remains limited, suggesting that low relaxed cost alone is not sufficient. Sharpening helps translate the learned connected structure into better decoded tours.

\begin{figure}[htbp]
  \centering
  \includegraphics[width=0.85\columnwidth]{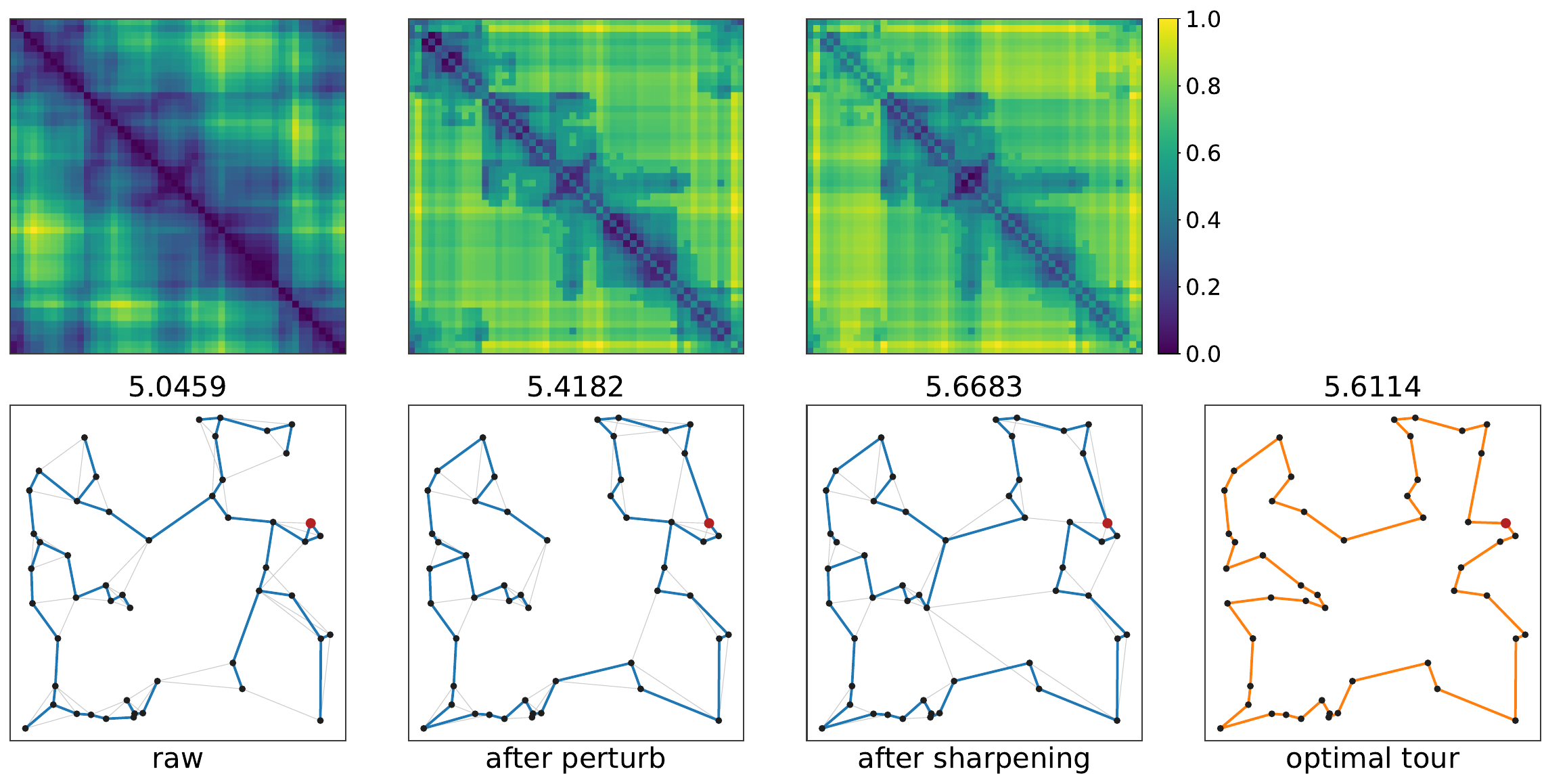}
  \caption{Refinement and decoding on a representative TSP50 instance. \textbf{Top}: Heatmaps of the raw cost matrix, the perturbed cost matrix, and the sharpened cost matrix, with nodes reordered according to the optimal tour so that optimal-tour edges appear near the super- and sub-diagonals. \textbf{Bottom}: the corresponding MAP rooted $1$-trees and the optimal tour. Titles report rooted $1$-tree or tour costs.}
  \label{fig:TSP50_structure_ins_v2}
\end{figure}

\begin{table}[htbp]
\centering
\caption{TSP50 ablation of structural and decoding metrics. Top2 Conc. measures concentration of each node's incident marginal mass on its two largest incident edges. Edge Cov. and Edge Cost Cov. report unweighted and cost-weighted recovery of Concorde-tour edges.}
\label{tab:tsp50_smooth_hk_metrics}
\begin{tabular}{lcccc}
\toprule
Method & Gap (\%) $\downarrow$ & Top2 Conc. $\uparrow$ & Edge Cov. $\uparrow$ & Edge Cost Cov. $\uparrow$ \\
\midrule
P + S & \textbf{1.55} & \textbf{0.9000} & \textbf{0.8751} & \textbf{0.8357} \\
P & 2.21 & 0.7438 & 0.8442 & 0.8006 \\
N-P & 19.25 & 0.1507 & 0.5845 & 0.5105 \\
N-P + S & 12.74 & 0.1923 & 0.6847 & 0.6046 \\
\bottomrule
\end{tabular}
\end{table}

\paragraph{Structural analysis.} Figure~\ref{fig:TSP50_structure_ins_v2} presents an example of a TSP50 instance that demonstrates the mechanisms of the proposed model across its refinement stages. Because nodes are reordered according to the optimal tour, low-cost mass near the super- and sub-diagonals indicates alignment with optimal-tour edges, whereas isolated low-cost regions far from the diagonal may induce structurally poor rooted $1$-trees. In the raw cost matrix, low-cost edges are dispersed across the edge space. Edge perturbation substantially reshapes this pattern by moving low-cost mass toward the tour-aligned diagonal band. Sharpening then makes a more localized adjustment: the heatmap before and after sharpening has a similar global pattern, but ambiguous off-diagonal low-cost regions are further suppressed. This is expected because sharpening acts as a fine-tuning step rather than a full re-learning step, which is also reflected in the MAP rooted $1$-tree, where only a small number of edges are changed. The resulting MAP rooted $1$-tree is therefore shaped toward a more tour-like structure, rather than merely assigned a lower cost.

\section{Conclusion}\label{sec:conclusion}
This work establishes connected latent support as a design principle for learning-based TSP. We introduced C2TSP, a differentiable rooted \(1\)-tree pipeline whose support is connected by construction, whose marginals are exactly computable, and whose remaining degree defects are corrected through a smoothed Held--Karp equilibration layer. The model learns residual edge perturbations directly from the original TSP cost, while certificate-guided sharpening reduces a provable upper bound on residual non-tour mass. In this way, C2TSP shifts the learned object from unconstrained edge scores to a tractable near-tour marginal with explicit structural meaning. Empirically, this representation yields strong pure-decoding performance and remains robust when used as a candidate source for LKH, where several heatmap-based baselines degrade at larger sizes. The main limitation of our study mainly follows from the same tractable surrogate, where the rooted \(1\)-trees relax exact Hamiltonian-cycle support as the design choice. C2TSP enforces connectivity in the support and restores degree balance in expectation, while the remaining degree fluctuations are controlled through equilibration and sharpening rather than eliminated exactly. This gap opens natural directions for tighter certificate-directed updates, richer tractable connected families, and scalable approximations to exact marginal computation. Overall, the results show that the differentiable latent object is a central modeling choice for learning-based TSP. By encoding connectedness before decoding, C2TSP exposes a near-tour marginal whose structural evolution can be trained, certified, and analyzed.

\clearpage
\bibliography{ref}

@inproceedings{amos2017optnet,
  title={OptNet: Differentiable Optimization as a Layer in Neural Networks},
  author={Amos, Brandon and Kolter, J. Zico},
  booktitle={Proceedings of the 34th International Conference on Machine Learning},
  pages={136--145},
  year={2017}
}

@inproceedings{agrawal2019diffcp,
  title={Differentiable Convex Optimization Layers},
  author={Agrawal, Akshay and Amos, Brandon and Barratt, Shane and Boyd, Stephen and Diamond, Steven and Kolter, J. Zico},
  booktitle={Advances in Neural Information Processing Systems},
  pages={9558--9570},
  year={2019}
}

@inproceedings{bai2019deq,
  title={Deep Equilibrium Models},
  author={Bai, Shaojie and Kolter, J. Zico and Koltun, Vladlen},
  booktitle={Advances in Neural Information Processing Systems},
  pages={688--699},
  year={2019}
}

@inproceedings{blondel2022modular,
  title={Efficient and Modular Implicit Differentiation},
  author={Blondel, Mathieu and Berthet, Quentin and Cuturi, Marco and Frostig, Roy and Hoyer, Stephan and Llinares-L{\'o}pez, Felipe and Pedregosa, Fabian and Vert, Jean-Philippe},
  booktitle={Advances in Neural Information Processing Systems},
  year={2022}
}

@inproceedings{kool2019attention,
  title={Attention, Learn to Solve Routing Problems!},
  author={Kool, Wouter and van Hoof, Herke and Welling, Max},
  booktitle={International Conference on Learning Representations},
  year={2019}
}

@inproceedings{kwon2020pomo,
  title={POMO: Policy Optimization with Multiple Optima for Reinforcement Learning},
  author={Kwon, Yeong-Dae and Choo, Jinho and Kim, Byoungjip and Yoon, Iljoo and Gwon, Youngjune and Min, Seungjai},
  booktitle={Advances in Neural Information Processing Systems},
  year={2020}
}

@inproceedings{mena2018gumbelsinkhorn,
  title={Learning Latent Permutations with Gumbel-Sinkhorn Networks},
  author={Mena, Gonzalo E. and Belanger, David and Linderman, Scott W. and Snoek, Jasper},
  booktitle={International Conference on Learning Representations},
  year={2018}
}

@inproceedings{min2023utsp,
  title={Unsupervised Learning for Solving the Travelling Salesman Problem},
  author={Min, Yimeng and Bai, Yiwei and Gomes, Carla P.},
  booktitle={Advances in Neural Information Processing Systems},
  year={2023}
}

@inproceedings{xin2021neurolkh,
  title={NeuroLKH: Combining Deep Learning Model with Lin-Kernighan-Helsgaun Heuristic for Solving the Traveling Salesman Problem},
  author={Xin, Liang and Song, Wen and Cao, Zhiguang and Zhang, Jie},
  booktitle={Advances in Neural Information Processing Systems},
  year={2021}
}

@article{held1970traveling,
  title={The traveling-salesman problem and minimum spanning trees},
  author={Held, Michael and Karp, Richard M},
  journal={Operations research},
  volume={18},
  number={6},
  pages={1138--1162},
  year={1970},
  publisher={INFORMS}
}

@article{held1971traveling,
  title={The traveling-salesman problem and minimum spanning trees: Part II},
  author={Held, Michael and Karp, Richard M},
  journal={Mathematical programming},
  volume={1},
  number={1},
  pages={6--25},
  year={1971},
  publisher={Springer}
}

@inproceedings{vinyals2015pointer,
  author    = {Vinyals, Oriol and Fortunato, Meire and Jaitly, Navdeep},
  title     = {Pointer Networks},
  booktitle = {Advances in Neural Information Processing Systems},
  volume    = {28},
  year      = {2015},
}

@article{bello2016neural,
  title={Neural combinatorial optimization with reinforcement learning},
  author={Bello, Irwan and Pham, Hieu and Le, Quoc V and Norouzi, Mohammad and Bengio, Samy},
  journal={arXiv preprint arXiv:1611.09940},
  year={2016}
}

@inproceedings{fu2021generalize,
  title={Generalize a small pre-trained model to arbitrarily large tsp instances},
  author={Fu, Zhang-Hua and Qiu, Kai-Bin and Zha, Hongyuan},
  booktitle={Proceedings of the AAAI conference on artificial intelligence},
  volume={35},
  number={8},
  pages={7474--7482},
  year={2021}
}

@article{qiu2022dimes,
  title={Dimes: A differentiable meta solver for combinatorial optimization problems},
  author={Qiu, Ruizhong and Sun, Zhiqing and Yang, Yiming},
  journal={Advances in Neural Information Processing Systems},
  volume={35},
  pages={25531--25546},
  year={2022}
}

@article{xia2024position,
  title={Position: Rethinking post-hoc search-based neural approaches for solving large-scale traveling salesman problems},
  author={Xia, Yifan and Yang, Xianliang and Liu, Zichuan and Liu, Zhihao and Song, Lei and Bian, Jiang},
  journal={arXiv preprint arXiv:2406.03503},
  year={2024}
}

@article{min2026graph,
  title={Graph Neural Networks are Heuristics},
  author={Min, Yimeng and Gomes, Carla P},
  journal={arXiv preprint arXiv:2601.13465},
  year={2026}
}

@article{paulus2020gradient,
  title={Gradient estimation with stochastic softmax tricks},
  author={Paulus, Max and Choi, Dami and Tarlow, Daniel and Krause, Andreas and Maddison, Chris J},
  journal={Advances in Neural Information Processing Systems},
  volume={33},
  pages={5691--5704},
  year={2020}
}

@article{sun2023difusco,
  title={Difusco: Graph-based diffusion solvers for combinatorial optimization},
  author={Sun, Zhiqing and Yang, Yiming},
  journal={Advances in neural information processing systems},
  volume={36},
  pages={3706--3731},
  year={2023}
}

@article{li2024fast,
  title={Fast t2t: Optimization consistency speeds up diffusion-based training-to-testing solving for combinatorial optimization},
  author={Li, Yang and Guo, Jinpei and Wang, Runzhong and Zha, Hongyuan and Yan, Junchi},
  journal={Advances in Neural Information Processing Systems},
  volume={37},
  pages={30179--30206},
  year={2024}
}

@incollection{applegate2011traveling,
  title={The traveling salesman problem: a computational study},
  author={Applegate, David L and Bixby, Robert E and Chv{\'a}tal, Va{\v{s}}ek and Cook, William J},
  booktitle={The traveling salesman problem},
  year={2011},
  publisher={Princeton university press}
}

@article{bother2022s,
  title={What's wrong with deep learning in tree search for combinatorial optimization},
  author={B{\"o}ther, Maximilian and Ki{\ss}ig, Otto and Taraz, Martin and Cohen, Sarel and Seidel, Karen and Friedrich, Tobias},
  journal={arXiv preprint arXiv:2201.10494},
  year={2022}
}

@article{parjadis2023learning,
  title={Learning Lagrangian Multipliers for the Travelling Salesman Problem},
  author={Parjadis, Augustin and Cappart, Quentin and Dilkina, Bistra and Ferber, Aaron and Rousseau, Louis-Martin},
  journal={arXiv preprint arXiv:2312.14836},
  year={2023}
}

@inproceedings{gharan2011randomized,
  title={A randomized rounding approach to the traveling salesman problem},
  author={Gharan, Shayan Oveis and Saberi, Amin and Singh, Mohit},
  booktitle={2011 IEEE 52nd Annual Symposium on Foundations of Computer Science},
  pages={550--559},
  year={2011},
  organization={IEEE}
}

@article{genova2017experimental,
  title={An experimental evaluation of the best-of-many Christofides’ algorithm for the traveling salesman problem},
  author={Genova, Kyle and Williamson, David P},
  journal={Algorithmica},
  volume={78},
  number={4},
  pages={1109--1130},
  year={2017},
  publisher={Springer}
}

\clearpage

\appendix

\section{Proofs}\label{app:proofs}

\subsection{Proof of Lemma~\ref{lem:tractable-rooted-map}}
\label{app:proof-tractable-rooted-map}
\begin{proof}
Let $r\in V$ be the fixed root and $\bar V=V\setminus\{r\}$. For non-root edges $\{i,j\}\subseteq\bar V$, define
\[
    w^\eta_{ij}
    =
    \exp\!\left(-\eta_{ij}/\tau\right),
\]
and for root edges define
\[
    a_i
    =
    \exp\!\left(-\eta_{ri}/\tau\right),
    \qquad i\in\bar V .
\]
Every rooted $1$-tree $U\in\Ur$ decomposes uniquely into a spanning tree $T$ on $\bar V$ and an unordered pair of distinct root neighbors $\{i,j\}\subseteq\bar V$. Therefore,
\begin{align}
Z_r(\eta)
&=
\sum_T
\sum_{i<j}
\exp\!\left(
-\frac{1}{\tau}
\Bigl[
\sum_{e\in T}\eta_e+\eta_{ri}+\eta_{rj}
\Bigr]
\right)
\notag\\
&=
\left(
\sum_T \prod_{e\in T} w^\eta_e
\right)
\left(
\sum_{i<j,\;i,j\in\bar V} a_i a_j
\right).
\end{align}
By the matrix-tree theorem,
\[
\sum_T \prod_{e\in T} w^\eta_e
=
\det L^\sharp(W),
\]
where $L^\sharp(W)$ is any cofactor of the weighted Laplacian on the non-root graph. Hence,
\[
Z_r(\eta)
=
\det L^\sharp(W)
\sum_{i<j,\;i,j\in\bar V} a_i a_j .
\]
This proves the partition factorization.

The same factorization gives exact edge marginals. For non-root edges, the spanning-tree marginal identity gives
\[
\mu_{ij}(\eta)
=
w^\eta_{ij} R^{\mathrm{eff}}_{ij}(W),
\qquad \{i,j\}\subseteq\bar V ,
\]
where $R^{\mathrm{eff}}_{ij}(W)$ is the effective resistance between $i$ and $j$ in the weighted non-root graph. For root edges, the marginal is the probability that $i$ belongs to the sampled unordered pair of root neighbors:
\[
\mu_{ri}(\eta)
=
\frac{a_i\sum_{j\in\bar V,\;j\neq i}a_j}
{\sum_{k<\ell,\;k,\ell\in\bar V}a_k a_\ell},
\qquad i\in\bar V .
\]
The computational cost is dominated by one Laplacian factorization on the non-root graph, so the partition function, edge marginals, and degree moments are computable exactly.
\end{proof}

\subsection{Proof of Theorem~\ref{thm:hk-equilibrium}}
\label{app:proof-hk-equilibrium}

\begin{proof}
For a fixed edge cost field $c\in\mathbb R^{|\Et|}$ and node prices
$\lambda\in\mathbb R^{n-1}$, define
\[
    E_{c,\lambda}(U)
    :=
    \langle c+A\lambda,x_U\rangle .
\]
By the node-to-edge lift identity,
\[
    \langle A\lambda,x_U\rangle
    =
    \sum_{i\in\bar V}\lambda_i d_i(U).
\]
Therefore the Lagrangian of~\eqref{eq:hk-primal}, up to the simplex
constraint $p\in\Delta(\Ur)$, is
\[
    \mathcal L(p,\lambda)
    =
    \sum_{U\in\Ur}p(U)E_{c,\lambda}(U)
    +
    \tau\sum_{U\in\Ur}p(U)\log p(U)
    -
    2\mathbf 1^\top\lambda .
\]
For fixed $\lambda$, minimizing over $p\in\Delta(\Ur)$ gives the Gibbs
distribution
\[
    p_\lambda(U)
    =
    \frac{1}{Z_r(c+A\lambda)}
    \exp\!\left(
        -\frac{1}{\tau}E_{c,\lambda}(U)
    \right)
    =
    q_{c+A\lambda}(U),
\]
and the minimized value is
\[
    -\tau\log Z_r(c+A\lambda)
    -
    2\mathbf 1^\top\lambda .
\]
Thus the Lagrange dual objective is
\[
    \Phi_\tau(c,\lambda)
    =
    -\tau\log Z_r(c+A\lambda)
    -
    2\mathbf 1^\top\lambda .
\]

We next compute its derivatives. Since
\[
    \frac{\partial E_{c,\lambda}(U)}{\partial\lambda_i}
    =
    d_i(U),
\]
differentiating the log-partition function gives
\[
    \frac{\partial\Phi_\tau(c,\lambda)}{\partial\lambda_i}
    =
    \mathbb E_{q_{c+A\lambda}}[d_i(U)]-2 .
\]
Differentiating once more,
\[
    \frac{\partial^2\Phi_\tau(c,\lambda)}
    {\partial\lambda_i\partial\lambda_j}
    =
    -\frac{1}{\tau}
    \operatorname{Cov}_{q_{c+A\lambda}}
    \bigl(d_i(U),d_j(U)\bigr).
\]
For any vector $a\in\mathbb R^{n-1}$,
\[
    a^\top\nabla^2_{\lambda\lambda}\Phi_\tau(c,\lambda)a
    =
    -\frac{1}{\tau}
    \operatorname{Var}_{q_{c+A\lambda}}
    \left(
        \sum_{i\in\bar V}a_i d_i(U)
    \right)
    \le 0 .
\]
Hence $\Phi_\tau(c,\lambda)$ is concave in $\lambda$. Therefore any
stationary maximizer $\lambda^\star(c)$ satisfies
\[
    0
    =
    \frac{\partial\Phi_\tau(c,\lambda^\star(c))}
    {\partial\lambda_i}
    =
    \mathbb E_{q_{c+A\lambda^\star(c)}}[d_i(U)]-2,
    \qquad i\in\bar V .
\]
This proves the HK degree-balance condition.
\end{proof}

\subsection{Proof of Theorem~\ref{thm:certificate}}
\label{app:proof-certificate}

\begin{proof}
For any \(U\in\Ur\), define
\[
    S(U):=\sum_{i\in\bar V}(d_i(U)-2)^2 .
\]
Every rooted $1$-tree has root degree two and \(n\) edges, hence
\[
    \sum_{i\in\bar V} d_i(U)=2(n-1),
    \qquad
    \sum_{i\in\bar V}(d_i(U)-2)=0 .
\]
If \(U\notin\Ht\), the integral deviations
\(\{d_i(U)-2\}_{i\in\bar V}\) are not all zero and sum to zero. Therefore at least one deviation is positive and at least one is negative, so
\[
    S(U)=\sum_{i\in\bar V}(d_i(U)-2)^2\ge 2 .
\]
Thus
\[
    \mathbf 1\{U\in\Nt\}
    \le
    \frac12 S(U).
\]
Taking expectation under \(q\) gives
\[
    \gamma(q)
    =
    q(\Nt)
    \le
    \frac12
    \sum_{i\in\bar V}
    \mathbb E_q[(d_i(U)-2)^2].
\]
Under the degree-balance condition \(\mathbb E_q[d_i(U)]=2\), this becomes
\[
    \gamma(q)
    \le
    \frac12
    \sum_{i\in\bar V}
    \operatorname{Var}_q(d_i(U)).
\]
For \(q^\star=q_{c+A\lambda^\star(c)}\), the Hessian identity follows directly
from~\eqref{eq:hk-hessian}:
\[
    \nabla_{\lambda\lambda}^2\Phi_\tau(c,\lambda^\star(c))
    =
    -\frac1\tau
    \operatorname{Cov}_{q^\star}(d(U),d(U)).
\]
Taking traces gives~\eqref{eq:hessian-certificate}.
\end{proof}

\subsection{Proof of Proposition~\ref{prop:oracle-edge-field}}
\label{app:proof-oracle-edge-field}

\begin{proof}
The Lagrangian of~\eqref{eq:oracle-recovery}, up to constants independent of
\(p\), is
\[
    \sum_{U\in\Ur}p(U)
    \left\langle
        D+A\bar\lambda+C^\top\nu,
        x_U
    \right\rangle
    +
    \tau\sum_{U\in\Ur}p(U)\log p(U).
\]
Stationarity in \(p\) gives
\[
    p^\star(U)
    =
    q_{D+A\bar\lambda^\star+C^\top\nu^\star}(U).
\]
Decompose the raw field as
\[
    C^\top\nu^\star=A\beta^\star+\Gamma^\star,
    \qquad
    \Gamma^\star\in\operatorname{Im}(A)^\perp,
\]
and set \(\lambda^\star:=\bar\lambda^\star+\beta^\star\). Then
\[
    D+A\bar\lambda^\star+C^\top\nu^\star
    =
    D+\Gamma^\star+A\lambda^\star,
\]
so \(p^\star=q_{\eta^\star}\) with
\[
    \eta^\star=D+\Gamma^\star+A\lambda^\star .
\]
The degree KKT conditions give
\[
    \mathbb E_{p^\star}[d_i(U)]=2,
    \qquad i\in\bar V .
\]
By~\eqref{eq:hk-gradient}, this is equivalent to
\[
    \nabla_\lambda\Phi_\tau(D+\Gamma^\star,\lambda^\star)=0.
\]
Concavity of \(\Phi_\tau\) in \(\lambda\) gives
\[
    \lambda^\star\in\argmax_\lambda\Phi_\tau(D+\Gamma^\star,\lambda).
\]
Finally,
\[
    \mu^\star
    =
    \mathbb E_{p^\star}[x_U]
    =
    \mathbb E_{q_{\eta^\star}}[x_U]
    =
    \mu(\eta^\star).
\]
\end{proof}

\subsection{Proof of Proposition~\ref{prop:ber-cert-descent}}
\label{app:proof-ber-cert-descent}

We use the following notation. Let
\(B:\mathbb R^{|\Et|}\to\mathbb R^{|\bar V|}\) be the non-root degree
operator, so that
\[
    Bx_U=d_{\bar V}(U),
    \qquad U\in\Ur .
\]
With the lift \(A\) defined in~\eqref{eq:lift}, we have \(A=B^\top\).
For an edge parameter \(\eta\), write
\[
    u(\eta):=\mu(\eta)=\mathbb E_{q_\eta}[x_U],
    \qquad
    \Sigma_x(\eta):=\operatorname{Cov}_{q_\eta}(x_U,x_U).
\]
When the parameter is clear, we write simply \(u\) and \(\Sigma_x\).

Recall
\[
    \mathcal C_{\rm Ber}(u)
    =
    \sum_{e\in\Et}w_eu_e(1-u_e),
    \qquad
    w_e:=\frac12\sum_{i\in\bar V}B_{ie}.
\]
Its marginal gradient is
\[
    b_{\rm Ber}(u)
    :=
    \nabla_u\mathcal C_{\rm Ber}(u)
    =
    w\odot(1-2u).
\]
The projected sharpening direction is
\[
    g_{\rm Ber}(u)
    :=
    P_\perp b_{\rm Ber}(u),
    \qquad
    P_\perp:=P_{\operatorname{Im}(A)^\perp}.
\]

\begin{lemma}[Bernoulli upper certificate]
\label{lem:app-ber-upper-cert}
Let \(q_\eta\) be HK-equilibrated and let \(u=\mu(\eta)\). Then
\[
    \gamma(q_\eta)
    \le
    \mathcal C_{\rm Ber}(u).
\]
\end{lemma}

\begin{proof}
For each \(i\in\bar V\),
\[
    d_i(U)
    =
    \sum_{e\in\Et}B_{ie}x_{U,e}.
\]
The rooted \(1\)-tree Gibbs law factorizes into a weighted spanning-tree law
on \(\bar V\) and an independent weighted two-edge root-selection law. Edge
indicators in both factors are negatively associated; hence, for \(e\ne f\),
\[
    \operatorname{Cov}_{q_\eta}(x_{U,e},x_{U,f})\le 0.
\]
Therefore,
\[
\begin{aligned}
    \operatorname{Var}_{q_\eta}(d_i(U))
    &=
    \operatorname{Var}_{q_\eta}
    \left(
        \sum_{e\in\Et}B_{ie}x_{U,e}
    \right)                                                   \\
    &=
    \sum_{e\in\Et}B_{ie}^2
    \operatorname{Var}_{q_\eta}(x_{U,e})
    +
    2\sum_{e<f}B_{ie}B_{if}
    \operatorname{Cov}_{q_\eta}(x_{U,e},x_{U,f})               \\
    &\le
    \sum_{e\in\Et}B_{ie}u_e(1-u_e),
\end{aligned}
\]
because \(B_{ie}\in\{0,1\}\). Combining this with
Theorem~\ref{thm:certificate} gives
\[
\begin{aligned}
    \gamma(q_\eta)
    &\le
    \frac12
    \sum_{i\in\bar V}
    \operatorname{Var}_{q_\eta}(d_i(U))                       \\
    &\le
    \frac12
    \sum_{i\in\bar V}
    \sum_{e\in\Et}
    B_{ie}u_e(1-u_e)                                          \\
    &=
    \sum_{e\in\Et}
    \left(
        \frac12\sum_{i\in\bar V}B_{ie}
    \right)
    u_e(1-u_e)                                                \\
    &=
    \mathcal C_{\rm Ber}(u).
\end{aligned}
\]
\end{proof}

\begin{lemma}[Tangent identity of HK re-equilibration]
\label{lem:app-hk-tangent-identity}
Fix an edge input \(c\) at which the local HK map
\(c\mapsto \lambda^\star(c)\) is differentiable. Define
\[
    \eta(c):=c+A\lambda^\star(c),
    \qquad
    T_c:=D\eta(c).
\]
Then
\[
    B\Sigma_xT_c=0,
    \qquad
    T_cv-v\in\operatorname{Im}(A)
    \quad
    \text{for every }v\in\mathbb R^{|\Et|},
\]
where \(\Sigma_x=\Sigma_x(\eta(c))\).
\end{lemma}

\begin{proof}
The HK equilibrium condition is
\[
    B\mu(c+A\lambda^\star(c))-2\mathbf 1=0.
\]
The rooted \(1\)-tree Gibbs marginal satisfies
\[
    D\mu(\eta)[v]
    =
    -\frac1\tau\Sigma_xv.
\]
Differentiating the HK equilibrium condition in direction \(v\) gives
\[
    B\Sigma_xT_cv=0.
\]
Since this holds for every \(v\),
\[
    B\Sigma_xT_c=0.
\]

The second identity follows directly from
\[
    \eta(c)=c+A\lambda^\star(c).
\]
Indeed, differentiating in direction \(v\) gives
\[
    T_cv
    =
    v+A\,D\lambda^\star(c)[v],
\]
and hence
\[
    T_cv-v
    =
    A\,D\lambda^\star(c)[v]
    \in\operatorname{Im}(A).
\]
\end{proof}

\begin{lemma}[Projected Bernoulli descent identity]
\label{lem:app-projected-ber-descent}
Define the reduced certificate
\[
    \overline{\mathcal C}_{\rm Ber}(c)
    :=
    \mathcal C_{\rm Ber}
    \left(
        \mu(c+A\lambda^\star(c))
    \right).
\]
Let
\[
    u=u(c):=\mu(c+A\lambda^\star(c)),
    \qquad
    \Sigma_x=\Sigma_x(c+A\lambda^\star(c)).
\]
Then
\[
    D\overline{\mathcal C}_{\rm Ber}(c)
    [g_{\rm Ber}(u)]
    =
    -
    \frac1\tau
    \left\|
        T_cg_{\rm Ber}(u)
    \right\|_{\Sigma_x}^2
    \le 0,
\]
where
\[
    \|v\|_{\Sigma_x}^2:=v^\top\Sigma_xv .
\]
\end{lemma}

\begin{proof}
Let
\[
    b=b_{\rm Ber}(u),
    \qquad
    g=g_{\rm Ber}(u).
\]
For any edge-input perturbation \(\delta c\),
\[
\begin{aligned}
    D\overline{\mathcal C}_{\rm Ber}(c)[\delta c]
    &=
    b^\top
    D\mu(c+A\lambda^\star(c))[T_c\delta c]              \\
    &=
    -
    \frac1\tau
    b^\top\Sigma_xT_c\delta c .
\end{aligned}
\]
Taking \(\delta c=g\), it remains to simplify
\(b^\top\Sigma_xT_cg\).

Because \(g=P_\perp b\), the removed component lies in
\(\operatorname{Im}(A)\). Thus there exists \(\beta\) such that
\[
    b-g=A\beta.
\]
Using \(A=B^\top\) and Lemma~\ref{lem:app-hk-tangent-identity},
\[
    (A\beta)^\top\Sigma_xT_cg
    =
    \beta^\top B\Sigma_xT_cg
    =
    0.
\]
Therefore,
\[
    b^\top\Sigma_xT_cg
    =
    g^\top\Sigma_xT_cg.
\]

By Lemma~\ref{lem:app-hk-tangent-identity},
\[
    T_cg-g\in\operatorname{Im}(A).
\]
Hence there exists \(\zeta\) such that
\[
    T_cg=g+A\zeta,
    \qquad
    \text{equivalently}
    \qquad
    g=T_cg-A\zeta.
\]
Again using \(A=B^\top\) and \(B\Sigma_xT_c=0\),
\[
    (A\zeta)^\top\Sigma_xT_cg
    =
    \zeta^\top B\Sigma_xT_cg
    =
    0.
\]
Thus
\[
\begin{aligned}
    g^\top\Sigma_xT_cg
    &=
    (T_cg-A\zeta)^\top\Sigma_xT_cg       \\
    &=
    (T_cg)^\top\Sigma_xT_cg              \\
    &=
    \|T_cg\|_{\Sigma_x}^2.
\end{aligned}
\]
Substituting into the directional derivative gives
\[
    D\overline{\mathcal C}_{\rm Ber}(c)[g]
    =
    -
    \frac1\tau
    \|T_cg\|_{\Sigma_x}^2
    \le 0.
\]
\end{proof}

The above three lemmas together allow us to formally close the proof for Proposition~\ref{prop:ber-cert-descent}, as summarized below. 

\begin{proof}[Proof of Proposition~\ref{prop:ber-cert-descent}]
The certificate claim follows from Lemma~\ref{lem:app-ber-upper-cert}.

For the descent claim, fix an edge input \(c\) at which the local HK map is
differentiable, and let
\[
    u(c)=\mu(c+A\lambda^\star(c)).
\]
By Lemma~\ref{lem:app-projected-ber-descent},
\[
    D\overline{\mathcal C}_{\rm Ber}(c)
    [g_{\rm Ber}(u(c))]
    =
    -
    \frac1\tau
    \left\|
        T_cg_{\rm Ber}(u(c))
    \right\|_{\Sigma_x}^2 .
\]
Therefore the derivative is strictly negative whenever the re-equilibrated
sharpening direction is nonzero, namely whenever
\[
    \left\|
        T_cg_{\rm Ber}(u(c))
    \right\|_{\Sigma_x}>0.
\]
By differentiability of
\(\overline{\mathcal C}_{\rm Ber}\), there exists \(\bar\alpha>0\) such that
for every \(0<\alpha<\bar\alpha\),
\[
    \overline{\mathcal C}_{\rm Ber}
    \left(
        c+\alpha g_{\rm Ber}(u(c))
    \right)
    <
    \overline{\mathcal C}_{\rm Ber}(c).
\]
With
\[
    c_\alpha=c+\alpha g_{\rm Ber}(u(c)),
    \qquad
    u_\alpha=\mu(c_\alpha+A\lambda^\star(c_\alpha)),
\]
this is exactly
\[
    \mathcal C_{\rm Ber}(u_\alpha)
    <
    \mathcal C_{\rm Ber}(u(c)).
\]
The proposition follows.
\end{proof}

\section{Training details}\label{app:training_details}

\subsection{Experimental environment}\label{app:experimental_env}
Experiments were conducted on a workstation equipped with an AMD Ryzen Threadripper PRO 7955WX v16 CPU and 1 NVIDIA RTX 5090, running Ubuntu \texttt{22.04} LTS. The models were implemented using Python \texttt{3.10} and PyTorch \texttt{2.7.0}, with additional libraries including torch-scatter \texttt{2.1.2} and torch-geometric \texttt{2.6.1}.
For reproducibility, a random seed of \texttt{42} was set for all random number generators. The code for this experiment is available at~\url{https://anonymous.4open.science/r/C2TSP-EF65}.

\subsection{Training procedure and implicit differentiation}
\label{app:training-implicit}

We summarize the training procedure of C2TSP and the implicit differentiation used to backpropagate through the smoothed HK equilibration layer. The node price $\lambda$ has a degenerate additive direction, so we resolve this degeneration by mapping $\lambda$ to a fixed nondegenerate subspace; for simplicity, we still write the resulting coordinate as $\lambda$.

\begin{theorem}[Implicit differentiation of the HK layer]
\label{thm:implicit-diff-hk}
Let
\[
F(\theta,\lambda)
:=
\mathbb E_{q_{D+\Gamma_\theta+A\lambda}}
[d_{\bar V}(U)]
-
2\mathbf 1 .
\]
Suppose that
\[
F(\theta,\lambda^\star(\theta))=0
\]
and that
\[
H(\theta,\lambda^\star)
:=
-\partial_\lambda F(\theta,\lambda^\star)
\]
is positive definite on the chosen nondegenerate subspace. Then $\lambda^\star(\theta)$ is locally unique and differentiable, and
\[
\frac{d\lambda^\star}{d\theta}
=
-
\bigl(\partial_\lambda F(\theta,\lambda^\star)\bigr)^{-1}
\partial_\theta F(\theta,\lambda^\star).
\]
Hence, the backward pass through the HK layer reduces to a single linear solve in the dimension of the restricted $\lambda$ subspace.
\end{theorem}

\begin{algorithm}[t]
\caption{Training C2TSP with residual rooted $1$-tree sharpening}
\label{alg:c2tsp_train}
\begin{algorithmic}[1]
\Require Training batch $(X,D)$, root $r$, GNN $f_\theta$, perturbation weight $\omega$, temperature $\tau$, sharpening weight $\alpha$, sharpening steps $K$, learning rate $\eta$
\Ensure Updated parameters $\theta$
\State $C^{\mathrm{base}} \gets D$
\State $L \gets f_\theta(X,D)$ \Comment{raw residual edge field}
\State $\widetilde L \gets \Pi_{\perp}(L)$ \Comment{remove node-additive component}
\State $C^{(0)} \gets C^{\mathrm{base}}-\omega \widetilde L$
\State $\lambda^{(0)} \gets \lambda^\star(C^{(0)})$ \Comment{HK equilibration}
\State $\eta^{(0)} \gets C^{(0)}+A\lambda^{(0)}$
\State $\mu^{(0)} \gets \mu(\eta^{(0)})$
\For{$k=0,\ldots,K-1$}
    \State $g^{(k)} \gets \Pi_{\perp}(1-2\mu^{(k)})$ \Comment{certificate-guided sharpening}
    \State $C^{(k+1)} \gets C^{(k)}+\alpha g^{(k)}$
    \State $\lambda^{(k+1)} \gets \lambda^\star(C^{(k+1)})$ \Comment{re-equilibration}
    \State $\eta^{(k+1)} \gets C^{(k+1)}+A\lambda^{(k+1)}$
    \State $\mu^{(k+1)} \gets \mu(\eta^{(k+1)})$
\EndFor
\State $\mathcal L_{\mathrm{cost}}(\theta) \gets \langle D,\mu^{(K)}\rangle$
\State $\theta \gets \theta-\eta\nabla_\theta \mathcal L_{\mathrm{cost}}(\theta)$
\State \Return $\theta$
\end{algorithmic}
\end{algorithm}

\subsection{Structured MAP decoding}
\label{app:map_covariance_decoder}

The structured MAP decoder converts the final rooted $1$-tree distribution into feasible tours for evaluation. Given the final sharpened cost $C^{(T)}$ and marginal $\mu^{(T)}$, the decoder repeatedly extracts MAP rooted $1$-trees, repairs them into Hamiltonian tours, and returns the candidate with the lowest true metric cost. Unlike local-search-based post-processing, this decoder directly uses the connected rooted $1$-tree structure and only applies a repair step to enforce the global connectivity constraint.

For each decode draw $k=0,\ldots,K-1$, we form a perturbed cost
\[
    C^{(k)} = C^{(T)}+\sigma \Delta C^{(k)},
\]
where $\Delta C^{(0)}=0$ gives the deterministic candidate and the remaining draws use small graph-structured perturbations. The perturbation is used only to diversify the MAP rooted $1$-tree candidates. Given $C^{(k)}$, the decoder computes
\[
    \mathcal T^{(k)}
    =
    \mathrm{MAPRootedOneTree}(C^{(k)},r),
\]
which consists of a minimum spanning tree on the non-root nodes together with the two cheapest root edges. Since $\mathcal T^{(k)}$ is connected but not necessarily a tour, it is repaired into a feasible Hamiltonian cycle:
\[
    \pi^{(k)}
    =
    \mathrm{Repair}(\mathcal T^{(k)},\mu^{(T)},C^{(k)},D).
\]
The repair step use the learned marginals, the perturbed cost, and the true metric distance matrix. Each candidate is evaluated by its true metric tour cost, and the decoder returns the best one:
\[
    \widehat{\pi}
    =
    \argmin_{\pi^{(k)}} \mathrm{TourCost}(\pi^{(k)},D).
\]

\begin{algorithm}[t]
\caption{Structured MAP decoding}
\label{alg:map_cov_decode}
\begin{algorithmic}[1]
\Require Final cost $C^{(T)}$, final marginal $\mu^{(T)}$, distance matrix $D$, root $r$, decode draws $K$, perturbation scale $\sigma$
\Ensure Decoded tour $\widehat{\pi}$
\For{$k=0,\ldots,K-1$}
  \State Generate a small perturbation $\Delta C^{(k)}$ with $\Delta C^{(0)}=0$
  \State $C^{(k)} \gets C^{(T)}+\sigma\Delta C^{(k)}$
  \State $\mathcal T^{(k)} \gets \mathrm{MAPRootedOneTree}(C^{(k)},r)$
  \State $\pi^{(k)} \gets \mathrm{Repair}(\mathcal T^{(k)},\mu^{(T)},C^{(k)},D)$
  \State $c^{(k)} \gets \mathrm{TourCost}(\pi^{(k)},D)$
\EndFor
\State $\widehat{\pi}\gets \argmin_{\pi^{(k)}}c^{(k)}$
\State \Return $\widehat{\pi}$
\end{algorithmic}
\end{algorithm}

\end{document}